\documentclass{article}
\usepackage{aditi}

\usepackage{microtype}
\usepackage{xspace}
\usepackage{hyperref}
\usepackage{url}
\usepackage{booktabs}

\usepackage{style}
\usepackage{algorithm}
\usepackage{algorithmic}
\usepackage[most]{tcolorbox}
\usepackage{colortbl}
\usepackage{tablefootnote}
\usepackage{threeparttable}
\usepackage{tabularx}
\usepackage{multirow}
\usepackage{makecell}
\usepackage{listings}
\usepackage{stfloats}

\definecolor{skyblue}{RGB}{204,229,255}
\tcbset{
  lightbluebox/.style={
    colback=skyblue!70,
    colframe=skyblue!75!black,
    boxrule=1pt,
    arc=4pt,
    auto outer arc
  }
}

\definecolor{TsinghuaPurple}{HTML}{660874}
\definecolor{mycolor_gray}{HTML}{ECECEC}
\definecolor{mydarkgreen}{RGB}{0,100,0}
\setContactAccent{TsinghuaPurple}
\hypersetup{colorlinks=true, citecolor=TsinghuaPurple, linkcolor=TsinghuaPurple, urlcolor=TsinghuaPurple}
\AtBeginDocument{\hypersetup{colorlinks=true, citecolor=TsinghuaPurple, linkcolor=TsinghuaPurple, urlcolor=TsinghuaPurple}}

\renewenvironment{abstract}{%
  \vskip.075in
  \begin{tcolorbox}[
    enhanced,
    colback=TsinghuaPurple!5!white,
    colframe=TsinghuaPurple!5!white,
    boxrule=0pt,
    arc=8pt,
    outer arc=8pt,
    left=6mm,
    right=6mm,
    top=4mm,
    bottom=4mm,
    boxsep=0pt
  ]
  \centerline{\large\bfseries Abstract}\vspace{0.3ex}%
}{%
  \par
  \end{tcolorbox}
  \vskip 1ex
}
\DeclareCaptionFormat{ForumCaption}{%
  \textbf{\color{TsinghuaPurple}#1#2}#3%
}
\newtheorem{finding}{Finding}

\newtcolorbox{mycustombox}[1]{
  enhanced,
  colback=black!5!white,
  colframe=black!75!white,
  boxrule=0.4pt,
  coltitle=white,
  title=#1,
  titlerule=0.4pt,
  fontupper=\small,
  fonttitle=\small,
  before upper={\par\smallskipamount},
  breakable,
  left=3mm,
  right=3mm,
  top=2mm,
  bottom=2mm,
  boxsep=1mm,
}

\setTitleruleGap{0.75pt}

\title{OmniVerifier-M1: Multimodal Meta-Verifier \\ with Explicit Structured Recalibration}

\setauthors{Xinchen Zhang$^{1*\dagger}$ \authorsep Bowei Liu$^{1*}$ \authorsep Jiale Liu$^{2}$ \authorsep Chufan Shi$^{3}$ \authorsep Yizhen Zhang$^{1}$ \authorsep Junhong Liu$^{4}$ \\ Youliang Zhang$^{1}$ \authorsep Zhiheng Li$^{1}$ \authorsep Yujiu Yang$^{1\ddagger}$ \authorsep Ling Yang$^{5\ddagger}$}
\setaffils{$^{1}$Tsinghua University \authorsep $^{2}$Pennsylvania State University \authorsep 
$^{3}$University of Southern California \\
$^{4}$Microcyto  \authorsep 
$^{5}$Princeton University\\
$^{*}$: Equal Contribution. \ \  $^{\dagger}$: Project Leader. \ \  $^{\ddagger}$: Corresponding Authors.
}

\setemaillinks{\ForumEmailLink{zhangxc24@mails.tsinghua.edu.cn}, \ \ForumEmailLink{yangling0818@163.com}}
\setcode{https://github.com/Cominclip/OmniVerifier}
\setwebsite{https://omniverifier.github.io/}

\begin{document}

\maketitle

\begin{abstract}
Visual outcomes are increasingly central to multimodal large language models, making reliable and fine-grained verification essential for scaling generalist foundation models. In this work, we investigate \textbf{\textit{multimodal meta-verification}}, which leverages verifier-generated rationales rather than decision-only signals, and explore how to effectively incorporate meta-verification feedback into multimodal verifier training. We identify two key findings. First, symbolic verifier outputs (e.g., bounding boxes) outperform textual explanations as meta-verification rationales, enabling efficient rule-based reinforcement learning rewards while avoiding reliance on model-based rewards from auxiliary judge models. Second, decoupling reinforcement learning objectives for binary judgment and meta-verification substantially outperforms joint reward optimization, due to intrinsic differences in output structure and learning dynamics. Based on these insights, we train \textbf{OmniVerifier-M1}, a generalist visual verifier leveraging symbolic meta-verification and decoupled reinforcement learning. OmniVerifier-M1 provides robust verification and fine-grained error localization, and further enables \textbf{M1-TTS}, a verifier-driven agentic generation system achieving dynamic region-level self-correction. This approach paves the way for more reliable, interpretable, and fine-grained multimodal verification, supporting safer and more controllable foundation model deployment.
\end{abstract}

\section{Introduction}
Current multimodal large language models (MLLMs) demonstrate powerful reasoning and generative capabilities in a variety of inference scenarios and reasoning modes \citep{guo2025seed1, seed2026seed1, comanici2025gemini, zhang2025generative}. Visual outcomes serve as a crucial bridge connecting multimodal understanding and generation, whether they are produced via agentic tool-use \citep{o3, zheng2025deepeyes} or through native generative processes \citep{liao2025mogao, gu2025thinkmorph, deng2025emerging}. In interleaved multimodal reasoning and interactive systems, enabling precise, fine-grained, and reliably evaluable verification of visual outcomes is a key requirement for scaling unified multimodal models and advancing generative intelligence.

Universal verification of visual outcomes remains at an early stage. Most existing image reward models \citep{xu2024visionreward, zhang2024itercomp}, such as RewardDance \citep{wu2025rewarddance} and UnifiedReward \citep{wang2025unified}, focus primarily on training and evaluation in traditional text-to-image generation scenarios. OmniVerifier \citep{zhang2025generative} marks an important step toward more general, world-modeling-oriented visual verification by leveraging reinforcement learning with binary (True/False) judgments of visual outcomes. However, feedback limited to binary decisions without supervision from detailed generative critiques can be coarse and uninformative, reducing the granularity needed for precise and effective judge model improvment \citep{shao2025deepseekmath, wang2026reward}.

In this work, we move beyond binary verifier judgments and examine the reliability of verifier-generated rationales and explanations, a process referred to as meta-verification \citep{shao2025deepseekmath, wang2026reward}. Instead of relying on decision-level signals, meta-verification operates at the level of explanations to guide the learning objective, yielding more informative and more restrictive feedback. In the investigation of how to improve \textbf{\textit{multimodal verifier}} training by integrating meta-verification feedback, this work identifies two core findings:

\textbf{Finding 1: Symbolic verifier outputs beat textual ones in meta-verification, enabling scalable and reliable rule-based RL rewards.}

Motivated by the highly structured and spatial nature of visual representations, we use symbolic outputs (e.g., bounding boxes or points) as rationales for meta-verification feedback when training the verifier, instead of relying on textual explanations. Textual rationales require additional judge models for evaluation, which slows down meta-verification feedback and increases the risk of reward hacking. In contrast, symbolic rationales provide a structured approximation of explanatory intent that can be directly assessed with explicit rules. Experiments show that in meta-verification training, symbolic rationales consistently match or outperform textual explanations, allowing rule-based feedback to replace model-based rewards, improving training efficiency while inherently preventing reward hacking.

\textbf{Finding 2: Decoupling RL rewards for binary judgment and meta-verification outperforms joint training in leveraging meta-verification feedback.}

In exploring how to better leverage meta-verification feedback for training the verifier, we find that combining binary judgment accuracy and meta-verification reward into a single joint reward for each sample offers little improvement in judgement accuracy. This is due to intrinsic differences in task structure and difficulty: binary judgments operate in a highly discrete output space, allowing the model to occasionally score well by chance, whereas meta-verification provides continuous, stronger supervision that effectively constrains such random behavior. To address this, we design a decoupling strategy that treats binary judgment and meta-verification as separate tasks with distinct reward systems for mixed data. Both empirical results and theoretical analysis confirm the superiority of decoupled training over joint training in the using of meta-verification.

Based on these observations, we train \textbf{OmniVerifier-M1}, a multimodal verifier adaptable to diverse multimodal foundation models \citep{cui2025emu3, cao2025hunyuanimage}. We adopt a decoupled training paradigm that leverages meta-verification feedback derived from symbolic outputs, enabling more effective and stable verifier optimization. Beyond serving as a multimodal visual verifier, OmniVerifier-M1 functions as a fine-grained multimodal optimizer that can precisely localize erroneous regions and provide actionable correction guidance. Building on this capability, we further develop a fine-grained multimodal agentic generation system, \textbf{M1-TTS}, in which verifier-driven decisions are expressed as heterogeneous, tool-level actions, including symbolic region localization and structured textual edit operations, and are iteratively coordinated through replanning to guide a unified foundation model toward region-level self-correction. Experimental results show that M1-TTS substantially outperforms conventional global-level multi-turn editing methods in correction effectiveness.

Our contributions can be summarized as follows:
\begin{itemize}
    \item \textbf{Multimodal Meta-Verification Paradigm:} We bring meta-verification to multimodal setting, enabling fine-grained verifier feedback beyond binary judgment.
    \item \textbf{Symbolic Meta-Verification Rationales:} We show that symbolic verifier outputs outperform textual explanations as meta-verification rationales, supporting efficient rule-based RL without reward hacking.
    \item \textbf{Decoupled Meta-Verification Training:} We theoretically and empirically demonstrate that decoupling reinforcement learning objectives for binary judgment and meta-verification substantially outperforms joint reward optimization.
    \item \textbf{Generalist Verifier and Agentic Correction System:} We develop \textbf{OmniVerifier-M1} and \textbf{M1-TTS}, a generalist multimodal verifier and an agentic correction system that support robust visual verification and effective region-level self-correction across diverse generative foundation models.
\end{itemize}

\section{Related Work}

\paragraph{Generative Veirifer or Reward Models.} Unlike traditional reward models that only output a scalar reward \citep{ouyang2022training, zhang2024generative, xu2023imagereward,luo2025unlocking,luo2026narrow}, generative verifiers provide interpretable, generative critiques, offering immense potential for scaling test-time computation or reinforcement learning \citep{zhang2025generative, liu2025inference, wang2026reward, yang2026hermesflow}. LLM- or VLM-as-a-Judge \citep{zhu2023judgelm, chen2025judgelrm, chen2024mllm} methods leverage the reasoning capabilities of large models to make evaluations more transparent and accurate, pioneering the use of foundation models as evaluators. DeepSeekMath-V2 \citep{shao2025deepseekmath} introduces meta-verification to assesses whether issues identified by the verifier indeed exist, which enhance verifier training by providing strict supervision. OmniVerifier \citep{zhang2025generative} identifies three fundamental atomic capabilities for verifying visual outcomes, marking a first step toward a general-purpose mutlimodal verifier for universal scenarios. However, exploration of multimodal verifiers is still in an early stage. Starting from the essence of visual representations, we develop a robust multimodal verifier training paradigm based on symbolic outputs with decoupled reinforcement learning.

\paragraph{Iterative Refinement for Visual Generation.} As we move towards more general visual generation scenarios, especially complex compositional generation \citep{zhang2024realcompo, yang2024mastering, yang2025chartmimic} or world-knowledge reasoning tasks \citep{wang2025genexam, hu2025multimodal, yang2026ureason}, it is difficult to achieve perfect results in a single attempt. Many approaches address this by combining a visual verifier with a generative model, employing a generate-reflect-refine loop to progressively improve generated images \citep{qin2025uni, jiang2025draco, huang2025interleaving, jaiswal2026iterative}. ReflectionFlow \citep{zhuo2025reflection} constructs large-scale dataset to perform reflection tuning on diffusion transformer to achieve multiround refinement. OmniVerifier-TTS \citep{zhang2025generative} bridge the image generation and edit within unifed multimodal models through the guidence of visual verifier. These methods optimize images from a high-level, macro perspective. However, erroneous regions are often small and can be easily confused with visually similar attributes, making precise, multi-dimensional control via textual descriptions challenging. To address this, we build an agentic generation system based on symbolic verifier outputs, allowing targeted, region-level corrections through efficient multi-round refinement.

\section{Problem Formulation}
We study reinforcement learning-based training of a pointwise multimodal verifier under the RLVR (Reinforcement Learning with Verifier Rewards) paradigm \citep{shao2024deepseekmath, guo2025deepseek}. Our goal is to train a verifier that not only determines whether a visual outcome satisfies the given prompt, but also produces transparent, fine-grained, and actionable critiques, providing reliable supervision for model reflection and refinement.

\begin{figure*}[t!]
    \centering
    \vspace{-2mm}
    \includegraphics[width=\linewidth]{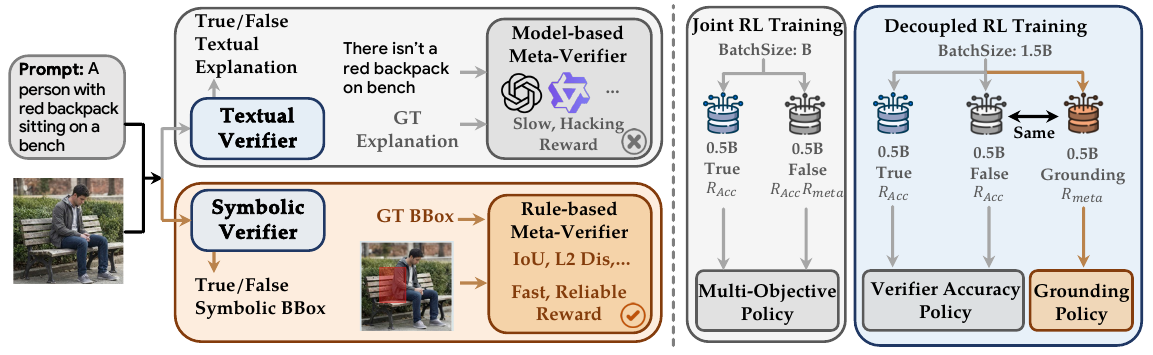}
    \caption{Pipeline of two key findings. Left: the advantage of symbolic bounding boxes over textual explanations, enabling rule-based rewards to inherently prevent reward hacking and accelerate training. Right: the comparison between joint training and decoupled training.}
    \vspace{-1mm}
    \label{fig-method}
\end{figure*}

\subsection{Baseline RLVR Training for Multimodal Verifiers}
Let $\mathcal{D} = \{(x_i, y_i)\}$ denote the training set, where $x_i = (I_i, P_i)$ consists of an image $I_i$ and its corresponding prompt $P_i$, and $y_i \in \{\texttt{True}, \texttt{False}\}$ is the ground-truth judgment of whether the image satisfies the prompt.

A visual verifier $\pi_\theta$ takes $(I, P)$ as input and generates a textual output $o$. A binary decision $\hat{y} \in \{\texttt{True}, \texttt{False}\}$ is then deterministically parsed from $o$ according to a predefined output format:
\begin{equation}
(o, \hat{y})  = \pi_\theta(I, P).
\end{equation}
The RL objective for training the verifier is:
\begin{equation}
\begin{aligned}
\max_{\pi_\theta} \;
\mathbb{E}_{\substack{
(I_i, P_i, y_i) \sim \mathcal{D},\\
(o_i, \hat{y}_i) \sim \pi_\theta(\cdot \mid I_i, P_i)
}}
\big[
\mathcal{R}_{\text{f}}(o_i)
+ \mathcal{R}_{\text{acc}}(\hat{y}_i, y_i)
\big].
\end{aligned}
\end{equation}
\paragraph{Format Reward.}
The format reward $\mathcal{R}_{\text{f}}(\cdot)$ requires the verifier ouput $o_i$ to perform an explicit reasoning step before giving the final judgment, where the verifier is instructed to include its intermediate analysis within \texttt{<think>} and \texttt{</think>} tags. The reward is realized as an indicator function checking strict adherence to this structure.
\paragraph{Accuracy Reward.}
The accuracy reward $\mathcal{R}_{\text{acc}}(\cdot)$ is a binary reward defined as
\begin{equation}
\mathcal{R}_{\text{acc}}(\hat{y}, y) =
\begin{cases}
1, & \text{if } \hat{y} = y, \\
0, & \text{otherwise}.
\end{cases}
\end{equation}
This reward provides supervision only at the decision level, without considering the correctness of the verifier’s reasoning or generative critique. While it can guide the model to learn coarse judgments, the learning signal is limited and easily exploitable: the model can achieve high reward by guessing or following superficial patterns, rather than performing meaningful verification. Consequently, this formulation fails to encourage fine-grained, interpretable, and reliable verification behavior.
\subsection{Meta-Verification Enhanced RLVR Training}
To overcome the limitations of decision-only supervision, meta-verification is used to enhance RLVR training of the verifier \citep{shao2025deepseekmath}. In this setting, the verifier is required to produce not only a binary decision, but also an explicit rationale when the decision is negative. Formally, the verifier outputs:
\begin{equation}
(o, \hat{y}, e) = \pi_\theta(I, P),
\end{equation}
where $\hat{y} \in \{\texttt{True}, \texttt{False}\}$ and $e$ denotes an explanation, which is only required when $\hat{y} = \texttt{False}$.

By integrating meta-verification feedback into the reward function, the enhanced verifier RL objective is formulated as:
\begin{equation}\label{eq5}
\max_{\pi_\theta} \mathbb{E}_{\substack{
(I_i, P_i, y_i) \sim \mathcal{D},\\
(o_i, \hat{y}_i, e_i) \sim \pi_\theta(\cdot \mid I_i, P_i)
}}
[
\mathcal{R}_{\text{f}}(o_i)
+
\mathcal{R}_{\text{acc}}(\hat{y}_i, y_i)
\cdot
\big(
\mathbb{I}[y_i=\texttt{True}]
+
\mathbb{I}[y_i=\texttt{False}]
\cdot
\mathcal{R}_{\text{meta}}(e_i)
\big)
].
\end{equation}
\paragraph{Meta-Verification Reward.}
The meta-verification reward $\mathcal{R}_{\text{meta}}(\cdot)$ evaluates the correctness and validity of the verifier-generated rationale $\hat{e}$. Specifically, a separate \textit{meta-verifier} $\mathcal{M}_\phi$ is used to assess whether the explanation correctly identifies genuine issues in the visual outcome:
\begin{equation}
\mathcal{R}_{\text{meta}} = \mathcal{M}_\phi(I, P, \hat{e}) \in \mathbb{R}.
\end{equation}
This reward provides supervision at the explanation level, encouraging the verifier to produce faithful and informative rationales rather than spurious or hallucinated justifications. By incorporating meta-verification feedback, the verifier receives denser and more restrictive learning signals that go beyond binary correctness, enabling improved reliability, interpretability, and training efficiency.

In subsequent sections, we further analyze how different forms of rationales and reward coupling strategies affect optimization dynamics, and show that symbolic rationales combined with decoupled reinforcement learning objectives yield substantially better performance.

\section{Symbolic Rationales for Rule-Based Multimodal Meta-Verification}
\label{ablation1}

\begin{figure*}[t]
    \centering
    \includegraphics[width=\linewidth]{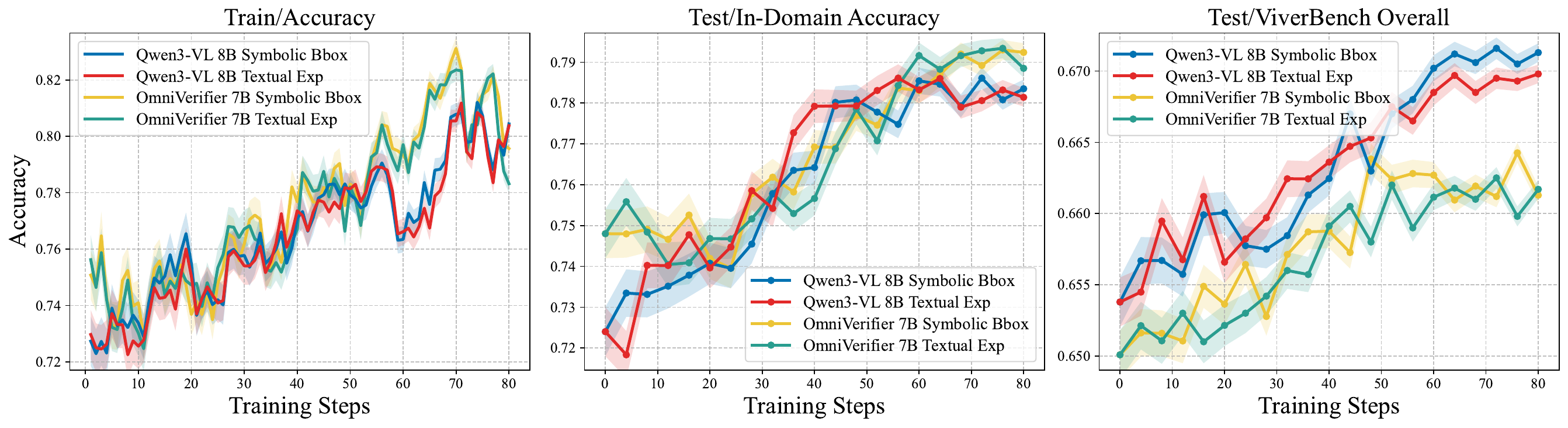}
    \caption{Comparison between symbolic bounding boxes and textual explanations as meta-verification signals in verifier RLVR training.}
    \label{fig-ablation1}
\end{figure*}


\paragraph{Drawbacks of Model-Based Meta-Verifiers.} Model-based reward models in RLVR aim to leverage the core capabilities of LLMs, particularly their advanced reasoning skills to produce more accurate judgments and rewards \citep{chen2025judgelrm, whitehouse2025j1}. Their flexibility mitigates the rigidity of rule-based rewards, which often struggle to generalize across diverse patterns. However, in dynamic reinforcement learning settings, these approaches are highly vulnerable to reward hacking: models may exploit weaknesses in the verifier to obtain high rewards without genuine improvements in reasoning, and in some cases even at the cost of degraded reasoning performance \citep{huang2025pitfallsrulemodelbasedverifiers, wang2026reward}. Moreover, applying model-based reward to large batches of samples generated during RL rollouts increases both the training cost and the overall training time \citep{wang2026reward}.

\paragraph{Revisiting Rule-Based Meta-Verifiers.} Beyond domains such as code and mathematics with structured answer, the diversity of output formats and the complexity of semantic composition make it difficult to directly apply rule-based signals as reinforcement learning rewards. In contrast, images constitute highly structured, spatially grounded, and high-dimensional representations. In visual outcome verification, errors in images are not only expressible through textual explanations, they can be captured through \textit{symbolic, structured outputs} such as bounding boxes, keypoints, or line segments that explicitly localize and characterize failure regions. For example, instead of generating verbose textual explanations, a verifier can output symbolic cues that spatially localize mismatched regions, providing concise and actionable feedback for correction, as shown in Fig. \ref{fig-method}. Such grounded symbolic feedback forms a natural basis for rule-based meta-verification, enabling precise error attribution without dependence on unconstrained textual reasoning.

\paragraph{Experimental Setup.} We apply DAPO \citep{yu2025dapo} to perform RL training on OmniVerifier-7B \citep{zhang2025generative, bai2025qwen2} and Qwen3-VL-8B \citep{Qwen3-VL}. For each training sample, we provide ground-truth binary judgments (True/False) together with ground-truth textual explanations and bounding boxes for meta-verification. For textual explanation, we use Qwen3-4B \citep{yang2025qwen3} to perform model-based comparation between the groundtruth explanation and verifier generated explanation to answer whether the two is semantically equal. For symbolic bounding box, we use intersection over union (IoU) as rule-based reward to provide meta-verification feedback. All the two models are trained for 80 steps on 16 NVIDIA A800-80G GPUs. We evaluate both models on ViVerBench \citep{zhang2025generative}, a comprehensive and challenging benchmark  designed for visual-outcome verification. 

\paragraph{Experimental Analysis.} 

\begin{table}[t]
\centering
\caption{
    Performance on ViVerBench and efficiency analysis.
}
\begin{threeparttable}
\resizebox{\linewidth}{!}{
\begin{tabular}{l|c|c|c|c|c}
\toprule
\textbf{Model} 
& \multicolumn{1}{c|}{\makecell{\textbf{ViVerBench} \\ \textbf{(Overall)}}} 
& \multicolumn{1}{c|}{\makecell{\textbf{Per-Card GPU} \\ \textbf{Memory (GB)}}} 
& \multicolumn{1}{c|}{\makecell{\textbf{Per-Sample Reward} \\ \textbf{Computation Time (ms)}}} 
& \multicolumn{1}{c|}{\makecell{\textbf{Training Time} \\ \textbf{per Step (min)}}}
& \multicolumn{1}{c}{\makecell{\textbf{Mean Response} \\ \textbf{Length (tokens)}}} \\
\midrule

\textbf{OmniVerifier 7B} 
& 0.650
& - 
& - 
& -
& - \\

\textbf{OmniVerifier 7B(Bbox)} 
& \cellcolor{mycolor_gray}{0.661}
&  \cellcolor{mycolor_gray}{48.6}
&  \cellcolor{mycolor_gray}{0.021}
&  \cellcolor{mycolor_gray}{8.13}
&  \cellcolor{mycolor_gray}{384}  \\

\textbf{OmniVerifier 7B(Exp)} 
& \cellcolor{mycolor_gray}{0.662} 
&  \cellcolor{mycolor_gray}{56.9}
&  \cellcolor{mycolor_gray}{20.2}
&  \cellcolor{mycolor_gray}{10.27}
&  \cellcolor{mycolor_gray}{340} \\

\textbf{Qwen 3-VL 8B} 
&  0.654
&  -
&  -
&  -
&  -\\

\textbf{Qwen 3-VL 8B(Bbox)} 
& \cellcolor{mycolor_gray}{0.671} 
& \cellcolor{mycolor_gray}{49.9} 
& \cellcolor{mycolor_gray}{0.021}
& \cellcolor{mycolor_gray}{8.74}
& \cellcolor{mycolor_gray}{516} \\

\textbf{Qwen 3-VL 8B(Exp)} 
& \cellcolor{mycolor_gray}{0.670} 
& \cellcolor{mycolor_gray}{58.3} 
& \cellcolor{mycolor_gray}{20.2}
& \cellcolor{mycolor_gray}{11.08}
& \cellcolor{mycolor_gray}{488} \\

\bottomrule
\end{tabular}
}
\end{threeparttable}
\label{tab:rule_based_vlm_viverbench_efficiency}
\end{table}

From Fig. \ref{fig-ablation1}, we observe that during training, the accuracy on the training set exhibits highly similar trends for both models, whether using symbolic bounding boxes or textual explanations as meta-verification signals. Moreover, their performance on both in-domain test sets and ViVerBench is also remarkably similar. This indicates that employing a rule-based IoU reward as meta-verification can serve as a reliable proxy for textual explanations. It effectively guides the verifier to improve its capabilities, while the symbolic format allows direct adherence to rule-based reward modeling, elegantly mitigating the issue of reward hacking at its source.


Additionally, as shown in Table~\ref{tab:rule_based_vlm_viverbench_efficiency}, we compare the computational overhead of rule-based and model-based meta-verification from both training and inference perspectives.  
During training, symbolic outputs show clear efficiency advantages over textual explanations by reducing GPU memory usage, per-sample reward computation time, and per-step training time, while maintaining comparable inference efficiency with similar response lengths.
Therefore, in multimodal verification scenarios, symbolic bounding box outputs can effectively replace textual explanations, providing comparable supervisory strength and inference-side overhead while substantially mitigating reward hacking and reducing training costs.

\begin{table}[!ht]
\begin{minipage}{\columnwidth}    
    \centering
    \begin{tcolorbox}[left=1em, right=1em]
        \small
    \vspace{-0.6em}
        \begin{finding}  
            Symbolic verifier outputs beat textual ones, unlocking rule-based RL rewards in meta-verification.
        \end{finding}
    \end{tcolorbox}
\end{minipage}
\vspace{-1em}
\end{table}

\section{Decoupled Reinforcement Learning Incentivizing Meta-Verification}

We investigate a general reinforcement learning paradigm for multimodal verifier training with meta-verification training. The formulation in Eq. \ref{eq5} as \textbf{\textit{joint training}}: for each training sample, we first assess the correctness of the binary judgment. When both the model prediction and the ground-truth label are \texttt{False}, we further employ a rule-based verifier (e.g., IoU) to generate meta-verification feedback.

A careful analysis of joint training reveals two intrinsic limitations. First, the meta-verification reward $\mathcal{R}_{\text{meta}}(\cdot)$ is activated only when both the prediction of the model and the ground-truth label are \texttt{False}, leading to a conditional and discontinuous gradient flow for the meta-verification objective. Second, binary judgment and meta-verification differ fundamentally in output structure and optimization landscape: the former operates over a discrete, low-entropy label space, while the latter requires learning continuous, fine-grained outputs. Jointly optimizing these heterogeneous objectives induces conflicting learning dynamics, which motivates an in-depth examination of the joint training paradigm:

\begin{figure*}[t]
    \centering
    \includegraphics[width=\linewidth]{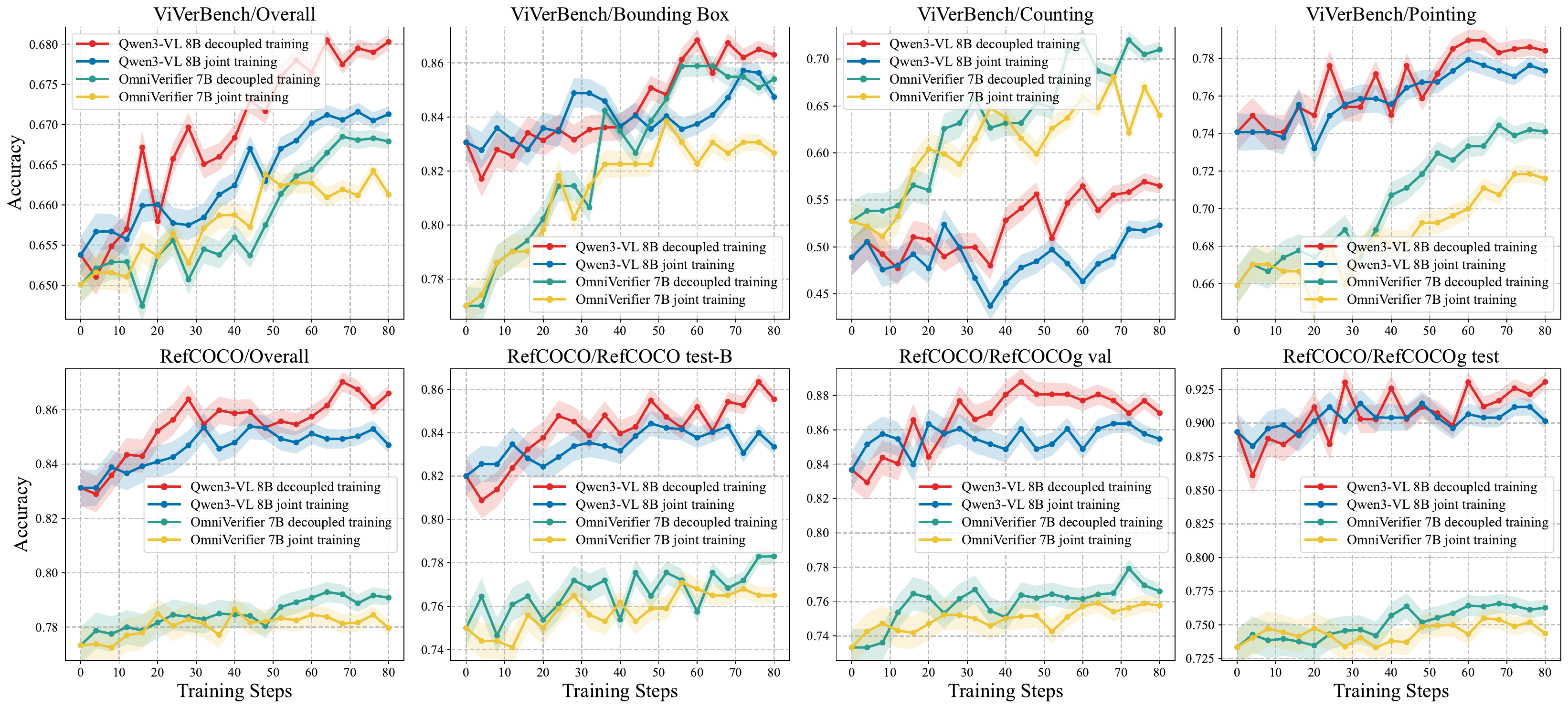}
    \caption{Performance comparison between decoupled and joint RL training with symbolic bounding boxes as meta-verification signals.}
    \label{fig-ablation2}
\end{figure*}


\begin{lemma} \label{lemma1}
    In joint RLVR training of multimodal verifier, all gradient terms related to the explanation $e$ are multiplicatively gated by the accuracy reward $\mathcal{R}_{\text{acc}}(\cdot)$.
\end{lemma}
All proofs are provided in Appendix \ref{appendix-proof}. This lemma reveals that under joint training, the verifier must first learn to make correct binary judgments before it can receive reward signals about where the error occurs. Based on this lemma, we have:

\begin{theorem} \label{theorem1}
Let the verifier’s decision (classification) accuracy on the data distribution be denoted as:
\begin{equation}
p_{\text{acc}}(\theta) = \mathbb{E}_{x \sim \mathcal{D}} \left[ \mathbb{P}_{\pi_\theta} \left( \hat{y} = y \mid x \right) \right],
\end{equation}
Then, in joint training, the gradients related to meta-verification satisfy:
\begin{equation}
\!\!\!\left\| \nabla_\theta J_{\text{joint}}^{(e)} \right\|\!\leq p_{\text{acc}}(\theta) \cdot \mathbb{E}\!\left[\left\| \mathcal{R}_{\text{meta}}(e) \nabla_\theta \log \pi_\theta(e \!\mid \!x, \hat{y}) \right\|\right].
\end{equation}
\end{theorem}
From this theorem, we observe that in the early stage of RL training, if $p_{\text{acc}}(\theta) \ll 1$, we have $\nabla_\theta J_{\text{joint}}^{(e)}\approx0$. This implies that meta-verification can hardly be optimized effectively. In particular, for smaller or less capable models, there exists an inherent gap between binary judgment and meta-verification.

Based on the above analysis of these limitations, we decompose binary judgment and meta-verification into two separate tasks, each served by an independent reward model, rather than coupling the two rewards in a sequential manner. We refer to this strategy as \textbf{\textit{decoupled training}}. Specifically, as shown in Fig. \ref{fig-method}, we start from original dataset $\mathcal{D} = \{(x, y)\}$, where positive and negative labels ($y=\texttt{True}$ and $y=\texttt{False}$) are balanced at a 1:1 ratio. The full dataset is used exclusively to supervise the accuracy reward $\mathcal{R}_{\text{acc}}(\cdot)$. In addition, we duplicate all samples with $y=\texttt{False}$, this duplicated subset is supervised solely by the meta-verification reward $\mathcal{R}_{\text{meta}}(\cdot)$. In this way, we explicitly decouple the verifier and meta-verification objectives at the dataset level and conduct mixed training across the two tasks.

We provide a detailed gradient-level analysis of both joint training and decoupled training:
\begin{theorem}
\label{theorem2}
Consider the gradient estimator for meta-verification in joint RLVR training:
\begin{equation}
\mathcal{G}_{\text{joint}}
=
\mathcal{R}_{\text{acc}}(x,\hat{y})
\cdot
\mathcal{R}_{\text{meta}}(e)
\cdot
\nabla_\theta \log \pi_\theta(e \mid x, \hat{y}),
\end{equation}
and the gradient estimator in decoupled training:
\begin{equation}
\mathcal{G}_{\text{dec}}
=
\mathcal{R}_{\text{meta}}(e)
\cdot
\nabla_\theta \log \pi_\theta(e \mid x, \hat{y}),
\end{equation}
where samples are drawn from the conditional distribution $x \sim \mathcal{D} \mid y=\text{False}$.
Then, the gradient variance in joint training satisfies:
\begin{equation}
\begin{aligned}
\mathrm{Var}(\mathcal{G}_{\text{joint}})
=&
p_{\text{acc}}(\theta)\mathrm{Var}(\mathcal{G}_{\text{dec}})
+\\
&p_{\text{acc}}(\theta)\left(1-p_{\text{acc}}(\theta)\right)
\left\|\mathbb{E}[\mathcal{G}_{\text{dec}}]\right\|^2,
\end{aligned}
\end{equation}
and consequently,
\begin{equation}
\mathrm{Var}(\mathcal{G}_{\text{joint}})
\;\ge\;
p_{\text{acc}}(\theta)\,\mathrm{Var}(\mathcal{G}_{\text{dec}}),
\end{equation}
with strict inequality when $\mathbb{E}[\mathcal{G}_{\text{dec}}]\neq 0$ and $p_{\text{acc}}(\theta)\in(0,1)$.
\end{theorem}

\begin{corollary}
\label{corollary1}
Let the signal-to-noise ratio (SNR) of a gradient estimator be defined as
\begin{equation}
\mathrm{SNR}(\mathcal{G})
=
\frac{\left\|\mathbb{E}[\mathcal{G}]\right\|^2}
{\mathrm{Var}(\mathcal{G})}.
\end{equation}
Then, the SNR of the meta-verification gradient under joint training satisfies:
\begin{equation}
\mathrm{SNR}(\mathcal{G}_{\text{joint}})
\le
p_{\text{acc}}(\theta)\,
\mathrm{SNR}(\mathcal{G}_{\text{dec}}),
\end{equation}
with strict inequality whenever $p_{\text{acc}}(\theta)\in(0,1)$.
\end{corollary}
Theorem~\ref{theorem2} shows that in joint RL training, the meta-verification gradient is effectively multiplied by a Bernoulli variable controlled by $p_{\text{acc}}(\theta)$, which both suppresses the expected gradient magnitude and introduces an additional variance term. Corollary~\ref{corollary1} further indicates this gating directly reduces the signal-to-noise ratio of the meta-verification gradient by a factor of $p_{\text{acc}}(\theta)$. Consequently, when the verifier’s judgment accuracy is imperfect, joint training yields sparse and noisy learning signals for meta-verification, whereas decoupled training removes this dependency and provides a more stable optimization signal.
\paragraph{Experimental Setup.} Following the same experimental setting as in Section~\ref{ablation1}, we evaluate the effectiveness of decoupled training in comparison with joint training. We decouple binary judgment and meta-verification into two independent learning objectives, each supervised by a dedicated reward model. Specifically, all samples with $y=\texttt{False}$ are duplicated and treated as grounding-only data, supervised exclusively by the meta-verification reward (e.g., IoU), while the remaining samples are supervised solely by the accuracy reward. These two data streams are jointly mixed during reinforcement learning. We evaluate all models on ViVerBench \citep{zhang2025generative} and RefCOCO \citep{yu2016modeling}, which respectively measure visual outcome judgment and visual grounding capability.

\begin{table*}[t]
\vspace{2mm}
\centering
\caption{
    Performance on ViVerBench of joint training and decoupled training.
}
\begin{threeparttable}
\resizebox{\linewidth}{!}{
\begin{tabular}{l|ccc|cc|cc|ccc|cccc|cc|l} 
\toprule
\multirow{2}{*}{\textbf{Model / Metric}} & \multicolumn{3}{c}{\textbf{Concept Existence}} & \multicolumn{2}{c}{\textbf{Object Relation}} & \multicolumn{2}{c}{\textbf{World Dynamics}} & \multicolumn{3}{c}{\textbf{Image Annotation}} & \multicolumn{4}{c}{\textbf{State Value Evaluation}} & \multicolumn{2}{c}{\textbf{STEM}} & \multirow{2}{*}{\textbf{Overall}} \\
\cmidrule(lr){2-4} \cmidrule(lr){5-6} \cmidrule(lr){7-8} \cmidrule(lr){9-11} \cmidrule(lr){12-15} \cmidrule(lr){16-17}
& Obj. & Attr. & Abs.P. & Spat. & N-Spat. & S-Phy & D-Phy & BBox & Point & Count & Maze & F.Lake & Robot. & GUI & Chart & LaTeX & \\ 
\midrule

\textbf{OmniVerifier 7B} & 0.701 &0.703 & 0.521 &0.808 &0.739 & 0.525 & 0.596 & 0.770 &0.659 & 0.527 & 0.490 & 0.482 & 0.728 & 0.634 & 0.600 & 0.918 & 0.650 \\

\textbf{OmniVerifier 7B(Joint)} 
& 0.723
& 0.733
& 0.513
& 0.833
& 0.761
& 0.487
& 0.564
& 0.827
& 0.716
& 0.640
& 0.497
& 0.436
& 0.601
& 0.694
& 0.623
& 0.928
& 0.661 \\

\textbf{OmniVerifier 7B(Decoupled)} 
& \cellcolor{mycolor_gray}{0.741}
& \cellcolor{mycolor_gray}{0.754}
& \cellcolor{mycolor_gray}{0.506}
& \cellcolor{mycolor_gray}{0.846}
& \cellcolor{mycolor_gray}{0.769}
& \cellcolor{mycolor_gray}{0.467}
& \cellcolor{mycolor_gray}{0.535}
& \cellcolor{mycolor_gray}{0.854}
& \cellcolor{mycolor_gray}{0.741}
& \cellcolor{mycolor_gray}{0.710}
& \cellcolor{mycolor_gray}{0.441}
& \cellcolor{mycolor_gray}{0.443}
& \cellcolor{mycolor_gray}{0.589}
& \cellcolor{mycolor_gray}{0.722}
& \cellcolor{mycolor_gray}{0.639}
& \cellcolor{mycolor_gray}{0.931}
& \cellcolor{mycolor_gray}{0.668} \\

\textbf{Qwen 3-VL 8B} & 0.710 & 0.690 &0.562 &0.642 & 0.716 & 0.604 & 0.593 & 0.831 & 0.741  &0.489  & 0.420&0.693  & 0.671 & 0.787 & 0.540 & 0.773 &0.654  \\
\textbf{Qwen 3-VL 8B(Joint)} 
& 0.732
& 0.724
& 0.534
& 0.704
& 0.754
& 0.595
& 0.582
& 0.847
& 0.773
& 0.523
& 0.458
& 0.664
& 0.639
& 0.815
& 0.568
& 0.824
& 0.671 \\

\textbf{Qwen 3-VL 8B(Decoupled)} 
& \cellcolor{mycolor_gray}{0.750}
& \cellcolor{mycolor_gray}{0.733}
& \cellcolor{mycolor_gray}{0.527}
& \cellcolor{mycolor_gray}{0.717}
& \cellcolor{mycolor_gray}{0.768}
& \cellcolor{mycolor_gray}{0.583}
& \cellcolor{mycolor_gray}{0.596}
& \cellcolor{mycolor_gray}{0.863}
& \cellcolor{mycolor_gray}{0.784}
& \cellcolor{mycolor_gray}{0.565}
& \cellcolor{mycolor_gray}{0.380}
& \cellcolor{mycolor_gray}{0.717}
& \cellcolor{mycolor_gray}{0.652}
& \cellcolor{mycolor_gray}{0.838}
& \cellcolor{mycolor_gray}{0.572}
& \cellcolor{mycolor_gray}{0.835}
& \cellcolor{mycolor_gray}{0.680} \\
\bottomrule
\end{tabular}
} 
\end{threeparttable}
\label{tab:rule_based_vlm_viverbench}
\end{table*}

\begin{table*}[t]
\centering
\caption{
    Performance on RefCOCO of joint training and decoupled training.
}
\begin{threeparttable}
\resizebox{\linewidth}{!}{
\begin{tabular}{l|cccccccc|c}
\toprule
\textbf{Model / Metric}
& RefCOCO val 
& RefCOCO test-A
& RefCOCO test-B
& RefCOCO+ val
& RefCOCO+ test-A
& RefCOCO+ test-B
& RefCOCOg val
& RefCOCOg test
& \textbf{Overall} \\
\midrule

\textbf{OmniVerifier 7B} 
& 0.807 & 0.890 & 0.750 & 0.757 & 0.810 & 0.707 & 0.733 & 0.733 & 0.773 \\

\textbf{OmniVerifier 7B (Joint)} 
& 0.810
& 0.897
& 0.765
& 0.760
& 0.773
& 0.733
& 0.758
& 0.744
& 0.780 \\

\textbf{OmniVerifier 7B (Decoupled)} 
& \cellcolor{mycolor_gray}{0.837}
& \cellcolor{mycolor_gray}{0.913}
& \cellcolor{mycolor_gray}{0.783}
& \cellcolor{mycolor_gray}{0.737}
& \cellcolor{mycolor_gray}{0.776}
& \cellcolor{mycolor_gray}{0.753}
& \cellcolor{mycolor_gray}{0.766}
& \cellcolor{mycolor_gray}{0.763}
& \cellcolor{mycolor_gray}{0.791} \\

\textbf{Qwen 3-VL 8B} 
& 0.860 &  0.883 &0.820 &0.747  & 0.867 & 0.743 &0.837  &  0.893& 0.831  \\

\textbf{Qwen 3-VL 8B (Joint)} 
& 0.887
& 0.910
& 0.834
& 0.763
& 0.873
& 0.753
& 0.855
& 0.901
& 0.847 \\

\textbf{Qwen 3-VL 8B (Decoupled)} 
& \cellcolor{mycolor_gray}{0.898}
& \cellcolor{mycolor_gray}{0.917}
& \cellcolor{mycolor_gray}{0.855}
& \cellcolor{mycolor_gray}{0.770}
& \cellcolor{mycolor_gray}{0.910}
& \cellcolor{mycolor_gray}{0.777}
& \cellcolor{mycolor_gray}{0.870}
& \cellcolor{mycolor_gray}{0.931}
& \cellcolor{mycolor_gray}{0.866} \\

\bottomrule
\end{tabular}
}
\end{threeparttable}
\label{tab:rule_based_vlm_refcoco}
\end{table*}

\paragraph{Experimental Analysis.} From Fig. \ref{fig-ablation2}, we observe that decoupled training consistently outperforms joint training on both OmniVerifier-7B and Qwen3-VL-8B. In particular, on ViVerBench tasks that are closely related to visual grounding, such as \textit{Bounding Box} , \textit{Counting}, and \textit{Pointing}, decoupled training yields substantially better performance. This improvement can be attributed to the more stable meta-verification gradients provided by decoupling, which enable the verifier to learn more precise and reliable grounding-oriented visual judgments. As further evidenced in Table \ref{tab:rule_based_vlm_refcoco}, models trained with the decoupled strategy also exhibit clear advantages on RefCOCO, demonstrating stronger visual grounding capability. These results indicate that the error localization ability learned by visual verifiers can effectively generalize to generic grounding tasks, rather than being confined to verifier-specific supervision. These findings suggest that, for visual verifier training, decoupled optimization constitutes a more robust and effective meta-verification reinforcement learning strategy than joint training, due to its ability to disentangle heterogeneous learning objectives and stabilize the training dynamics.
\begin{table}[h!]
\begin{minipage}{\columnwidth}    
    \centering
    \begin{tcolorbox}[left=1em, right=1em]
        \small
    \vspace{-0.6em}
        \begin{finding}  
            Decoupling RL rewards for binary judgment and meta-verification outperforms joint training.
        \end{finding}
    \end{tcolorbox}
\end{minipage}
\vspace{-1em}
\end{table}

\begin{figure*}[t]
    \centerline{\includegraphics[width=.95\textwidth]{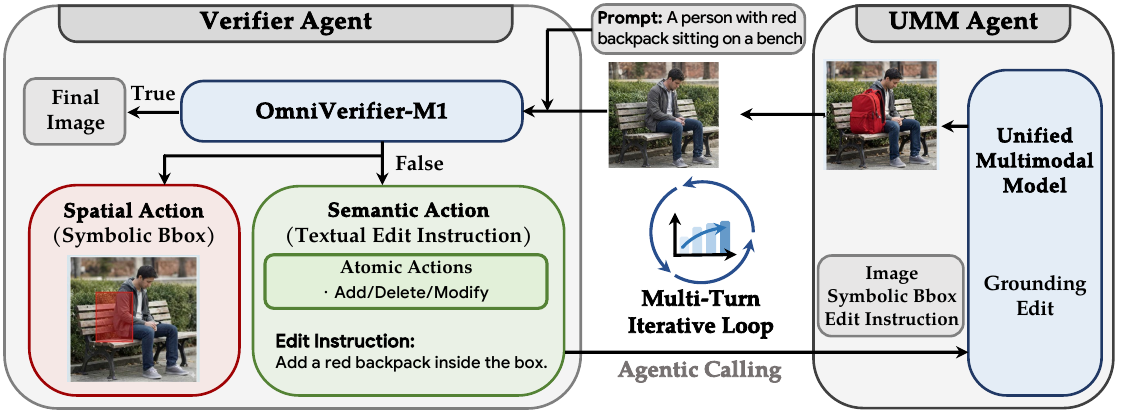}}
    \caption{Pipeline of M1-TTS: a fine-grained agentic generation system for unified multimodal models.}
    \label{m1-tts}
\end{figure*}

\section{Multimodal Verifier for Agentic Generation}
\subsection{OmniVerifier-M1: Generalist Multimodal Verifier}

Based on the above two experimental findings, we train OmniVerifier-M1, a generalist multimodal verifier built on Qwen3-VL-8B \citep{Qwen3-VL}, which uses symbolic bounding boxes as output rationales and leverages rule-based meta-verification reward feedback through decoupled reinforcement training.

As shown in the last row of Table \ref{tab:rule_based_vlm_viverbench}, OmniVerifier achieves a score of 0.68 on ViVerBench \citep{zhang2025generative}, with notable gains on key text-to-image verification tasks such as \textit{Object}, \textit{Attribute},  \textit{Spatial Relationship}, and \textit{Bounding Box}. This approach also significantly reduces training overhead and demonstrates the potential of a robust reinforcement learning paradigm for integrating meta-verification into verifier training as a generalizable framework.

\subsection{M1-TTS: Fine-Grained Agentic Generation}

OmniVerifier-M1 provides fine-grained, actionable feedback that precisely localizes erroneous regions in images, rather than offering only coarse, text-level explanations from a global perspective. Building on this capability, we design \textbf{M1-TTS}, an agentic generation system that leverages a visual verifier and a unified multimodal model (UMM) to perform fine-grained, precise, and high-difficulty image world modeling tasks. As shown in Fig. \ref{m1-tts}, M1-TTS consists of two components: a Verifier Agent and a UMM Agent. 

\paragraph{Verifier Agent.} The Verifier Agent serves as both the evaluator and diagnostician of input images. Given the current image whether newly generated by the UMM or edited from the previous iteration, if it is not aligned with the input prompt, OmniVerifier-M1 produces a structured action composed of two parts:

\begin{figure*}[t!]
    \centerline{\includegraphics[width=\textwidth]{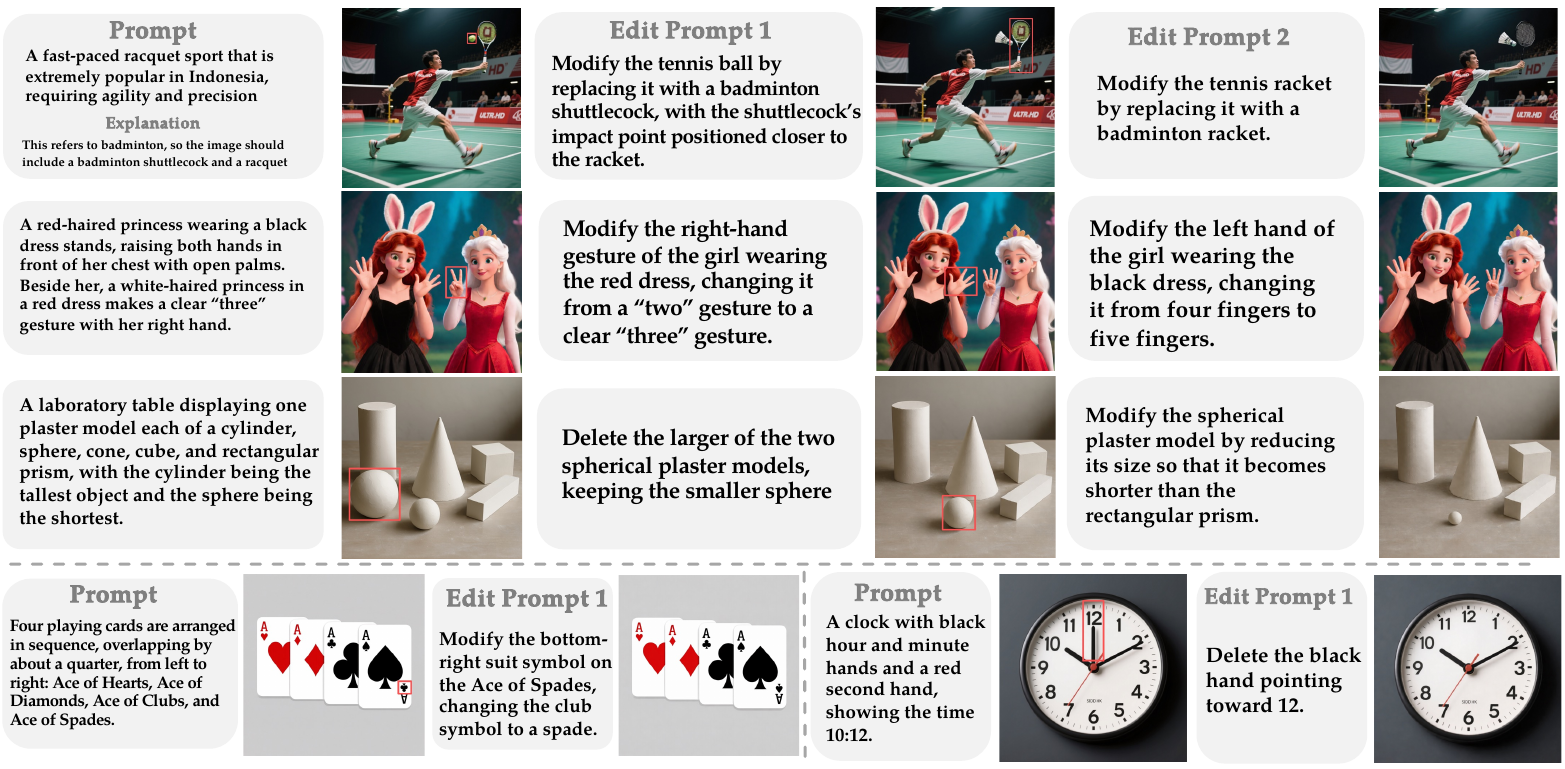}}
    \caption{Qualitative visualization of M1-TTS.}
    \label{fig-exp-tts}
\end{figure*}

\vspace{-3mm}
\begin{itemize}
    \item \textbf{Spatial Action.} OmniVerifier-M1 generates symbolic bounding boxes to precisely localize erroneous regions in the image. By explicitly identifying where corrections are needed, this spatial signal directly simplifies the UMM’s perception and reasoning burden, allowing a smooth and effective transition from error diagnosis to targeted image editing.
    \item \textbf{Semantic Action.} In addition to erroneous localization, OmniVerifier-M1 provides explicit semantic editing instructions. We predefine a set of atomic editing actions such as \textit{Add}, \textit{Delete}, and \textit{Modify}, and require OmniVerifier-M1 to compose accurate, actionable edit commands grounded in these atomic operations. This design enforces structured, interpretable refinement steps
\end{itemize}
\paragraph{UMM Agent.} The Unified Multimodal Model (UMM) performs image editing by taking as input the image, symbolic bboxes, and explicit editing instructions. The use of symbolic bboxes eliminates the need for the UMM to first fully parse complex editing instructions and then infer the corresponding spatial regions during editing, thereby significantly simplifying the editing process and improving editing precision.

M1-TTS supports dynamic multi-round image optimization, iteratively refining the image until the verifier outputs \texttt{True} or a maximum number of iterations is reached. Within M1-TTS, the verifier not only injects prior world knowledge into the UMM’s self-refinement process through its strong reasoning and judgment capabilities, but also compensates for the UMM’s limitations in visual perception and spatial localization. As a result, the verifier serves as a critical supervisory and guiding component that enables accurate, fine-grained, and iterative image refinement.
\vspace{-3mm}

\paragraph{Experimental Setup.} We conduct M1-TTS experiments on two strong generative models, RePlan \citep{wu2025qwen, qu2025replan} and GPT-Image-1.5, with the maximum number of iterative steps set to 10. We evaluate M1-TTS on WISE \citep{niu2025wise} and T2I-CoreBench \citep{li2025easier} to assess its capabilities in world-knowledge–driven generation and complex image generation, respectively. All experiments are conducted on 8 NVIDIA A800 GPUs
\vspace{-3mm}

\begin{table*}[t]
\centering
\caption{
Performance comparison on WISE and T2I-CoreBench.
}
\begin{threeparttable}
\resizebox{\linewidth}{!}{
\begin{tabular}{l|ccccccc|ccc}
\toprule
\multirow{2}{*}{\textbf{Model}} 
& \multicolumn{7}{c|}{\textbf{WISE}} 
& \multicolumn{3}{c}{\textbf{T2I-CoreBench}} \\
\cmidrule(lr){2-8} \cmidrule(lr){9-11}
& Cultural & Time & Space & Biology & Physics & Chemistry & Overall
& Composition & Reasoning & Overall \\
\midrule

\textbf{Qwen-Image}
& 0.62 & 0.63 & 0.77 & 0.57 & 0.75 & 0.40 & 0.62
& 0.780 & 0.493 & 0.589 \\

\textbf{Qwen3-VL 8B+RePlan}
& 0.66 & 0.64 & 0.77 & 0.56 & 0.72 & 0.47 & 0.65
& 0.793 & 0.547 & 0.629 \\

\textbf{OmniVerifier-M1+RePlan}
& \cellcolor{mycolor_gray}{0.71}
& \cellcolor{mycolor_gray}{0.61}
& \cellcolor{mycolor_gray}{0.77}
& \cellcolor{mycolor_gray}{0.56}
& \cellcolor{mycolor_gray}{0.74}
& \cellcolor{mycolor_gray}{0.57}
& \cellcolor{mycolor_gray}{0.68}
& \cellcolor{mycolor_gray}{0.817}
& \cellcolor{mycolor_gray}{0.626}
& \cellcolor{mycolor_gray}{0.690} \\

\textbf{GPT-Image-1.5}
& 0.89 & 0.69 & 0.88 & 0.80 & 0.76 & 0.78 & 0.83
& 0.855 & 0.746 & 0.782 \\

\textbf{Qwen3-VL 8B+GPT-Image-1.5}
& 0.88 & 0.80 & 0.91 & 0.89 & 0.85 & 0.81 & 0.86
& 0.857 & 0.752 & 0.787 \\

\textbf{OmniVerifier-M1+GPT-Image-1.5}
& \cellcolor{mycolor_gray}{0.90}
& \cellcolor{mycolor_gray}{0.80}
& \cellcolor{mycolor_gray}{0.92}
& \cellcolor{mycolor_gray}{0.91}
& \cellcolor{mycolor_gray}{0.88}
& \cellcolor{mycolor_gray}{0.83}
& \cellcolor{mycolor_gray}{0.88}
& \cellcolor{mycolor_gray}{0.863}
& \cellcolor{mycolor_gray}{0.769}
& \cellcolor{mycolor_gray}{0.800} \\
\bottomrule
\end{tabular}
}
\end{threeparttable}
    \vspace{-4mm}
\label{tab:wise_t2i_corebench}
\end{table*}
\paragraph{Experimental Analysis.} 
As shown in Fig. \ref{fig-exp-tts}, M1-TTS perform dynamic image generation both accurately and efficiently. For regions that are misaligned with the prompt or contain severe errors, OmniVerifier-M1 provides precise guidance through bounding boxes and explanatory signals. This is especially important for complex images with objects that share similar attributes, where bounding boxes more effectively highlight issues. Furthermore, as reported in Table \ref{tab:wise_t2i_corebench}, M1-TTS achieves substantial improvements on world knowledge (WISE) and complex text-to-image benchmarks (T2I-CoreBench), whether using RePlan or GPT-Image-1.5. These gains stem from OmniVerifier-M1’s ability to inject world knowledge while leveraging its interactive bbox outputs to guide the generative model in refining the image.

\section{Conclusion}

In this work, we present Multimodal Meta-Verification, a framework that extends verifier training beyond binary judgments by leveraging symbolic localization feedback. We identify two key findings: (1) Symbolic outputs provide structured and efficient rationales that outperform textual explanations, enabling rule-based reinforcement learning while mitigating reward hacking; (2)  Decoupled RL objectives for binary judgment and meta-verification facilitates more robust and efficient optimization, yielding consistently higher verification accuracy than joint training. Building on these insights, we develop a generalist multimodal verifier OmniVerifier-M1 and further introduce M1-TTS, a verifier-driven agentic generation system that performs region-level self-correction through iterative reasoning and action. We provide a robust training paradigm for multimodal meta-verification and leave the exploration of broader applications of verifiers to future work.

\section*{Impact Statements}
This paper introduces OmniVerifier-M1 and M1-TTS, and presents a robust framework for training multimodal verifiers with meta-verification feedback.

The methods and findings in this work are intended to advance research in efficient and scalable machine learning systems. We do not anticipate immediate negative societal impacts beyond those commonly associated with deploying more capable and efficient language model systems.

\newpage
\bibliography{example_paper}
\bibliographystyle{colm2025_conference}

\newpage
\appendix
\onecolumn
\section{Theoretical Proof}\label{appendix-proof}
\subsection{Proof of Lemma~\ref{lemma1}}
\begin{proof}
The expected objective for joint training is defined as:
\begin{equation}
    J_{\text{joint}}(\theta) = \mathbb{E}_{x \sim \mathcal{D}} \mathbb{E}_{(o,\hat{y},e) \sim \pi_\theta(\cdot|x)} \left[ \mathcal{R}_{\text{acc}}(\hat{y}, y) \cdot \mathcal{R}_{\text{meta}}(e) \right],
\end{equation}
The policy gradient of the joint objective can be written as:
\begin{equation}\label{eq-16}
    \nabla_\theta J_{\text{joint}} = \mathbb{E}\left[ \mathcal{R}_{\text{acc}}(\hat{y}, y) \cdot \mathcal{R}_{\text{meta}}(e) \cdot \nabla_\theta \log \pi_\theta(o, \hat{y}, e \mid x) \right].
\end{equation}
Since the prediction $\hat{y}$ and the meta-verification output $e$ are jointly generated by the same policy, the joint log-probability can be factorized as:
\begin{equation}
    \log \pi_\theta(o, \hat{y}, e \mid x) = \log \pi_\theta(\hat{y} \mid x) + \log \pi_\theta(e \mid x, \hat{y}).
\end{equation}
Substituting this decomposition into Eq.~\ref{eq-16}, we obtain:
\begin{align}
\nabla_\theta J_{\text{joint}}
&= \mathbb{E}\Big[
\mathcal{R}_{\text{acc}}(\hat{y}, y) \cdot \mathcal{R}_{\text{meta}}(e)
\cdot \nabla_\theta \big(
\log \pi_\theta(\hat{y} \mid x) + \log \pi_\theta(e \mid x, \hat{y})
\big)
\Big] \\
&= \mathbb{E}\Big[
\mathcal{R}_{\text{acc}}(\hat{y}, y) \cdot \mathcal{R}_{\text{meta}}(e)
\cdot \nabla_\theta \log \pi_\theta(\hat{y} \mid x)
\Big] + \mathbb{E}\Big[
\mathcal{R}_{\text{acc}}(\hat{y}, y) \cdot \mathcal{R}_{\text{meta}}(e)
\cdot \nabla_\theta \log \pi_\theta(e \mid x, \hat{y})
\Big].
\end{align}

The gradient terms related to explanations:
\begin{equation}
    \nabla_\theta J_{\text{joint}}^{(e)} = \mathbb{E}_{x \sim \mathcal{D}, (\hat{y}, e) \sim \pi_\theta} \left[ \mathcal{R}_{\text{acc}}(\hat{y}, y) \cdot \mathcal{R}_{\text{meta}}(e) \nabla_\theta \log \pi_\theta(e \mid x, \hat{y}) \right].
\end{equation}
Note that:
\begin{equation}
\mathcal{R}_{\text{acc}}(\hat{y}, y) = \mathbf{1}\left[\hat{y} = y\right].
\end{equation}

Therefore, when the verifier makes an incorrect prediction (i.e., $\hat{y} \neq y$):
\begin{equation}
\mathcal{R}_{\text{acc}}(\hat{y}, y) = 0 \quad \Rightarrow \quad \nabla_\theta J_{\text{joint}}^{(e)} = 0.
\end{equation}
\end{proof}
\subsection{Proof of Theorme~\ref{theorem1}}
\begin{proof}
According to the definition of the joint training objective, the gradient term $\nabla_\theta J_{\text{joint}}^{(e)}$ with respect to the explanation generation component $e$ can be expressed as:
\begin{equation}
    \nabla_\theta J_{\text{joint}}^{(e)} = \mathbb{E}_{x \sim \mathcal{D}, (\hat{y}, e) \sim \pi_\theta} \left[ \mathcal{R}_{\text{acc}}(\hat{y}, y) \cdot \mathcal{R}_{\text{meta}}(e) \nabla_\theta \log \pi_\theta(e \mid x, \hat{y}) \right].
\end{equation}
Considering that the accuracy reward $\mathcal{R}_{\text{acc}}(\hat{y}, y)$ is an indicator function $\mathbf{1}[\hat{y} = y]$, we apply Jensen's inequality to the $\ell_2$-norm of the above gradient, and the derivation proceeds as follows:
\begin{align}
\|\nabla_\theta J_{\text{joint}}^{(e)}\| 
&= \left\| \mathbb{E}_{x \sim \mathcal{D}} \left[ \mathbb{E}_{(\hat{y}, e) \sim \pi_\theta} \left[ \mathbf{1}[\hat{y} = y] \cdot \mathcal{R}_{\text{meta}}(e) \nabla_\theta \log \pi_\theta(e \mid x, \hat{y}) \right] \right] \right\| \\
&\le \mathbb{E}_{x \sim \mathcal{D}} \left[ \mathbb{E}_{(\hat{y}, e) \sim \pi_\theta} \left[ \left\| \mathbf{1}[\hat{y} = y] \cdot \mathcal{R}_{\text{meta}}(e) \nabla_\theta \log \pi_\theta(e \mid x, \hat{y}) \right\| \right] \right] \\
&= \mathbb{E}_{x \sim \mathcal{D}} \left[ \mathbb{E}_{\hat{y} \sim \pi_\theta(\cdot|x)} \left[ \mathbb{E}_{e \sim \pi_\theta(\cdot|x, \hat{y})} \left[ \mathbf{1}[\hat{y} = y] \cdot \left\| \mathcal{R}_{\text{meta}}(e) \nabla_\theta \log \pi_\theta(e \mid x, \hat{y}) \right\| \right] \right] \right].
\end{align}
Using the property of conditional expectation, when $\hat{y} \neq y$, the indicator function $\mathbf{1}[\hat{y} = y] = 0$, so only the case where $\hat{y} = y$ needs to be retained in the expectation term:
\begin{align}
\|\nabla_\theta J_{\text{joint}}^{(e)}\|
&\le \mathbb{E}_{x \sim \mathcal{D}} \left[ \mathbb{P}_{\pi_\theta}(\hat{y} = y \mid x) \cdot \mathbb{E}_{e \sim \pi_\theta(\cdot \mid x, y)} \left[ \left\| \mathcal{R}_{\text{meta}}(e) \nabla_\theta \log \pi_\theta(e \mid x, y) \right\| \right] \right] \\
&\le \mathbb{E}_{x \sim \mathcal{D}} \left[ \mathbb{P}_{\pi_\theta}(\hat{y} = y \mid x) \right] \cdot \sup_{x} \mathbb{E}_{e \sim \pi_\theta} \left[ \left\| \mathcal{R}_{\text{meta}}(e) \nabla_\theta \log \pi_\theta(e \mid x, y) \right\| \right] \\
&= p_{\text{acc}}(\theta) \cdot C,
\end{align}
where $C = \sup_{x} \mathbb{E}_{e} \left[ \| \mathcal{R}_{\text{meta}}(e) \nabla_\theta \log \pi_\theta(e \mid x, y) \| \right]$ denotes the finite upper bound of the gradient term. Based on the definition of the verification accuracy $p_{\text{acc}}(\theta) = \mathbb{E}_{x \sim \mathcal{D}} [\mathbb{P}_{\pi_\theta}(\hat{y} = y \mid x)]$, the proof is complete.
\end{proof}

\subsection{Proof of Throrme \ref{theorem2}}

\begin{proof}
Let $I = \mathcal{R}_{\text{acc}}(x, \hat{y}) = \mathbf{1}[\hat{y} = y]$ be the indicator variable representing the accuracy of the verifier. By definition, $I$ follows a Bernoulli distribution with parameter $p_{\text{acc}}(\theta) = \mathbb{P}(\hat{y} = y)$, such that $\mathbb{E}[I] = p_{\text{acc}}(\theta)$ and $I^2 = I$. The joint gradient estimator can be written as $\mathcal{G}_{\text{joint}} = I \cdot \mathcal{G}_{\text{dec}}$.

First, we calculate the first and second moments of $\mathcal{G}_{\text{joint}}$:
\begin{align}
\mathbb{E}[\mathcal{G}_{\text{joint}}] 
&= \mathbb{E}[I \cdot \mathcal{G}_{\text{dec}}] = p_{\text{acc}}(\theta) \mathbb{E}[\mathcal{G}_{\text{dec}}], \\
\mathbb{E}[\|\mathcal{G}_{\text{joint}}\|^2] 
&= \mathbb{E}[\|I \cdot \mathcal{G}_{\text{dec}}\|^2] = \mathbb{E}[I^2 \cdot \|\mathcal{G}_{\text{dec}}\|^2] = p_{\text{acc}}(\theta) \mathbb{E}[\|\mathcal{G}_{\text{dec}}\|^2].
\end{align}
Using the variance identity $\mathrm{Var}(Z) = \mathbb{E}[\|Z\|^2] - \|\mathbb{E}[Z]\|^2$ for a vector-valued random variable $Z$, we have:
\begin{align}
\mathrm{Var}(\mathcal{G}_{\text{joint}}) 
&= \mathbb{E}[\|\mathcal{G}_{\text{joint}}\|^2] - \|\mathbb{E}[\mathcal{G}_{\text{joint}}]\|^2 \\
&= p_{\text{acc}}(\theta) \mathbb{E}[\|\mathcal{G}_{\text{dec}}\|^2] - \|p_{\text{acc}}(\theta) \mathbb{E}[\mathcal{G}_{\text{dec}}]\|^2 \\
&= p_{\text{acc}}(\theta) \left( \mathrm{Var}(\mathcal{G}_{\text{dec}}) + \|\mathbb{E}[\mathcal{G}_{\text{dec}}]\|^2 \right) - p_{\text{acc}}(\theta)^2 \|\mathbb{E}[\mathcal{G}_{\text{dec}}]\|^2 \\
&= p_{\text{acc}}(\theta) \mathrm{Var}(\mathcal{G}_{\text{dec}}) + \left( p_{\text{acc}}(\theta) - p_{\text{acc}}(\theta)^2 \right) \|\mathbb{E}[\mathcal{G}_{\text{dec}}]\|^2 \\
&= p_{\text{acc}}(\theta) \mathrm{Var}(\mathcal{G}_{\text{dec}}) + p_{\text{acc}}(\theta)(1 - p_{\text{acc}}(\theta)) \|\mathbb{E}[\mathcal{G}_{\text{dec}}]\|^2.
\end{align}
Since $p_{\text{acc}}(\theta) \in [0, 1]$, the term $p_{\text{acc}}(\theta)(1 - p_{\text{acc}}(\theta)) \|\mathbb{E}[\mathcal{G}_{\text{dec}}]\|^2$ is always non-negative. Therefore:
\begin{equation}
\mathrm{Var}(\mathcal{G}_{\text{joint}}) \ge p_{\text{acc}}(\theta) \mathrm{Var}(\mathcal{G}_{\text{dec}}).
\end{equation}
The equality holds if and only if $p_{\text{acc}}(\theta)(1 - p_{\text{acc}}(\theta)) \|\mathbb{E}[\mathcal{G}_{\text{dec}}]\|^2 = 0$. Given $p_{\text{acc}}(\theta) \in (0, 1)$ and $\mathbb{E}[\mathcal{G}_{\text{dec}}] \neq 0$, the inequality is strict.
\end{proof}

\subsection{Proof of Corollary \ref{corollary1}}
\begin{proof}
Based on the results from Theorem \ref{theorem2}, we have the following expressions for the first moment and the variance of the joint gradient estimator:
\begin{align}
\|\mathbb{E}[\mathcal{G}_{\text{joint}}]\|^2 &= p_{\text{acc}}(\theta)^2 \|\mathbb{E}[\mathcal{G}_{\text{dec}}]\|^2, \\
\mathrm{Var}(\mathcal{G}_{\text{joint}}) &= p_{\text{acc}}(\theta) \mathrm{Var}(\mathcal{G}_{\text{dec}}) + p_{\text{acc}}(\theta)(1 - p_{\text{acc}}(\theta)) \|\mathbb{E}[\mathcal{G}_{\text{dec}}]\|^2.
\end{align}
Substituting these into the definition of $\mathrm{SNR}(\mathcal{G}_{\text{joint}})$, we obtain:
\begin{align}
\mathrm{SNR}(\mathcal{G}_{\text{joint}}) 
&= \frac{\|\mathbb{E}[\mathcal{G}_{\text{joint}}]\|^2}{\mathrm{Var}(\mathcal{G}_{\text{joint}})} \\
&= \frac{p_{\text{acc}}(\theta)^2 \|\mathbb{E}[\mathcal{G}_{\text{dec}}]\|^2}{p_{\text{acc}}(\theta) \mathrm{Var}(\mathcal{G}_{\text{dec}}) + p_{\text{acc}}(\theta)(1 - p_{\text{acc}}(\theta)) \|\mathbb{E}[\mathcal{G}_{\text{dec}}]\|^2} \\
&= \frac{p_{\text{acc}}(\theta) \|\mathbb{E}[\mathcal{G}_{\text{dec}}]\|^2}{\mathrm{Var}(\mathcal{G}_{\text{dec}}) + (1 - p_{\text{acc}}(\theta)) \|\mathbb{E}[\mathcal{G}_{\text{dec}}]\|^2}.
\end{align}
Since $1 - p_{\text{acc}}(\theta) \ge 0$, it follows that the denominator satisfies:
\begin{equation}
\mathrm{Var}(\mathcal{G}_{\text{dec}}) + (1 - p_{\text{acc}}(\theta)) \|\mathbb{E}[\mathcal{G}_{\text{dec}}]\|^2 \ge \mathrm{Var}(\mathcal{G}_{\text{dec}}).
\end{equation}
By applying this inequality to the denominator of the SNR expression, we have:
\begin{align}
\mathrm{SNR}(\mathcal{G}_{\text{joint}}) 
&\le \frac{p_{\text{acc}}(\theta) \|\mathbb{E}[\mathcal{G}_{\text{dec}}]\|^2}{\mathrm{Var}(\mathcal{G}_{\text{dec}})} \\
&= p_{\text{acc}}(\theta) \cdot \mathrm{SNR}(\mathcal{G}_{\text{dec}}).
\end{align}
For $p_{\text{acc}}(\theta) \in (0, 1)$, the term $(1 - p_{\text{acc}}(\theta)) \|\mathbb{E}[\mathcal{G}_{\text{dec}}]\|^2$ is strictly positive (assuming a non-vanishing signal $\|\mathbb{E}[\mathcal{G}_{\text{dec}}]\| > 0$), which makes the denominator strictly larger than $\mathrm{Var}(\mathcal{G}_{\text{dec}})$, thus confirming the strict inequality.
\end{proof}


\newpage

\section{Additional Experiments}

\subsection{Symbolic Point as Meta-Verification Signals}
We replace the symbolic bounding box with a symbolic point as the rule-based reward signal. In the bounding-box setting, for each negative sample, we compute the IoU between the predicted and ground-truth boxes and apply a threshold of 0.6 to obtain a binary gated reward, rather than using the continuous IoU value directly.

For consistency, we define the point-based reward in the same binary form. Specifically, if the predicted point falls inside the ground-truth bounding box, the error region is considered correctly localized and a reward of 1 is assigned; otherwise, the reward is set to 0.
As shown in Table~\ref{tab:symbolic_point_meta_verification}, rule-based symbolic point rewards also serve as an effective alternative to model-based textual explanations for meta-verification under the joint training setting.

\begin{table}[t]
\centering
\caption{
Performance Comparison of Rule-Based Symbolic Rewards and Model-Based Textual Rewards as Meta-Verification Signals on ViVerBench
}
\begin{threeparttable}
\resizebox{.7\linewidth}{!}{
\begin{tabular}{l|c|c}
\toprule
\textbf{Backbone} 
& \textbf{Reward Metric} 
& \textbf{ViVerBench} \\
\midrule

\textbf{OmniVerifier 7B} 
& -
& 0.6501 \\

\textbf{OmniVerifier 7B} 
& textual explanation (model-based)
& 0.6617 \\

\textbf{OmniVerifier 7B} 
& symbolic bbox (rule-based)
& 0.6613 \\

\textbf{OmniVerifier 7B} 
& symbolic point (rule-based)
& 0.6619 \\

\textbf{Qwen 3-VL 8B} 
& -
& 0.6539 \\

\textbf{Qwen 3-VL 8B} 
& textual explanation (model-based)
& 0.6698 \\

\textbf{Qwen 3-VL 8B} 
& symbolic bbox (rule-based)
& 0.6717 \\

\textbf{Qwen 3-VL 8B} 
& symbolic point (rule-based)
& 0.6709 \\

\bottomrule
\end{tabular}
}
\end{threeparttable}
\label{tab:symbolic_point_meta_verification}
\end{table}

\subsection{Evaluation of the Verifier's Localization Accuracy}
To directly evaluate the verifier’s ability to localize errors, we carefully construct a test set of 400 False samples that are not used during training, including 200 synthetic samples and 200 real-world samples. 
We compute the IoU between predicted and ground-truth bounding boxes and use a threshold of 0.6 consistent with the training setting to determine whether the error is successfully localized.

As shown in Table \ref{tab:error_localization_capability}, the significant improvement demonstrates that our decoupled symbolic rule-based RL effectively teaches the verifier to perform precise spatial error localization. Such high-precision localization capability ensures that the UMM agent receives reliable and fine-grained guidance

\begin{table}[t]
\centering
\caption{
    Evaluation of error localization capability on synthetic and real-world data.
}
\begin{threeparttable}
\resizebox{.7\linewidth}{!}{
\begin{tabular}{l|c|c}
\toprule
\textbf{Backbone} 
& \textbf{Synthetic Data} 
& \textbf{Real-World Data} \\
\midrule

\textbf{OmniVerifier 7B} 
& 0.290 
& 0.265 \\

\textbf{OmniVerifier 7B (Joint)} 
& 0.545
& 0.495 \\

\textbf{OmniVerifier 7B (Decoupled)} 
& 0.710
& 0.670 \\

\textbf{Qwen 3-VL 8B} 
& 0.375
& 0.325 \\

\textbf{Qwen 3-VL 8B (Joint)} 
& 0.665
& 0.605 \\

\textbf{Qwen 3-VL 8B (Decoupled)} 
& 0.780
& 0.725 \\

\bottomrule
\end{tabular}
}
\end{threeparttable}
\label{tab:error_localization_capability}
\end{table}

\subsection{Impact of Batch Size on Decoupled Training versus Joint Training}

\begin{table}[t]
\centering
\caption{
    Analysis of batch size on performance on ViVerBench and RefCOCO.
}
\begin{threeparttable}
\resizebox{.7\linewidth}{!}{
\begin{tabular}{l|c|c|c}
\toprule
\textbf{Backbone} 
& \textbf{Batch Size} 
& \textbf{ViVerBench} 
& \textbf{RefCOCO} \\
\midrule

\textbf{OmniVerifier 7B} 
& -
& 0.6501
& 0.7734 \\

\textbf{OmniVerifier 7B (Joint)} 
& 1B
& 0.6610
& 0.7800 \\

\textbf{OmniVerifier 7B (Decoupled)} 
& 1B
& 0.6672
& 0.7898 \\

\textbf{OmniVerifier 7B (Joint)} 
& 1.5B
& 0.6617
& 0.7813 \\

\textbf{OmniVerifier 7B (Decoupled)} 
& 1.5B
& 0.6680
& 0.7910 \\

\textbf{Qwen3-VL 8B} 
& -
& 0.6539
& 0.8313 \\

\textbf{Qwen3-VL 8B (Joint)} 
& 1B
& 0.6710
& 0.8470 \\

\textbf{Qwen3-VL 8B (Decoupled)} 
& 1B
& 0.6792
& 0.8642 \\

\textbf{Qwen3-VL 8B (Joint)} 
& 1.5B
& 0.6708
& 0.8473 \\

\textbf{Qwen3-VL 8B (Decoupled)} 
& 1.5B
& 0.6800
& 0.8660 \\

\bottomrule
\end{tabular}
}
\end{threeparttable}
\label{tab:batch_size_analysis}
\end{table}

To further verify whether the observed performance gains come from the training strategy rather than the batch size, we conduct additional experiments on both OmniVerifier and Qwen3-VL-8B. Specifically, (i) we increase the batch size of joint training to $1.5B$ and compare it with decoupled training under the same batch size; and (ii) we reduce the batch size of decoupled training to $B$ and compare it with joint training under the same setting.

As shown in Table~\ref{tab:batch_size_analysis}, decoupled training consistently outperforms joint training under the same batch size. This indicates that the performance gains do not simply come from using a larger batch size, but instead stem from the decoupled optimization strategy itself. 

In joint training, although the batch size is $B$, the same $0.5B$ negative samples are simultaneously used to optimize both the judgment objective and the grounding objective, meaning that each sample contributes to two different supervision signals. In contrast, in decoupled training, although the total batch size is $1.5B$, the same $0.5B$ negative samples are explicitly separated into objective-specific supervision: one part is used for the judgment objective, while the other is used for the grounding objective. Therefore, decoupled training changes how supervision signals are applied without increasing data diversity. 

Joint training suffers from sparse and entangled reward signals, whereas decoupled training reduces interference between optimization objectives and provides denser and more stable learning signals, leading to better performance.

\section{Data Construction Pipeline}
 we further provide a more detailed description of our two automated data construction pipelines.Our training data is constructed through two automated pipelines, ensuring that every false sample is associated with a meaningful and well-defined bounding box. \textbf{Importantly, our training data is entirely derived from OmniVerifier \citep{zhang2025generative}, enabling a fair comparison with the verifier and thereby allowing us to clearly demonstrate the advantages of meta-verification.} We construct the dataset using both synthetic data (ShareGPT-4o-Image) and real-world data (LVIS) through two automated methods to obtain both aligned and misaligned image–text pairs.
 
\paragraph{Method 1: Image-fixed, Prompt-modified.}
 For each complex image, we first use GPT-5 to generate a detailed prompt, which serves as the true prompt. We then modify the prompt using GPT-5 by adding or removing objects, altering attributes, or modifying spatial relationships to construct a mismatched (false) prompt, while GPT-5 simultaneously generates the corresponding ground-truth bounding boxes for the regions associated with these modifications.
 
 \paragraph{Method 2: Prompt-fixed, Image-inpainting.}
 We treat each complex image as the true image and first apply SAM 2.1 to segment it, obtaining masks and bounding boxes for all objects. To balance dataset difficulty, we dynamically select one object based on its mask area. We then perform inpainting using the selected mask to remove the object, thereby constructing a false image. Finally, we use GPT-5 to generate a detailed prompt from the true image, which serves as the fixed prompt. This construction naturally yields accurate and meaningful bounding boxes.

\section{Limitations and Future Works}
Despite the strong capabilities demonstrated by OmniVerifier-M1 in multimodal verifier training and M1-TTS in dynamic agentic generation, there remain two limitations that need to be addressed:

\begin{itemize}
    \item The verifier training paradigm proposed in this work requires validation on larger-scale backbone models and backbones with different architectures. Larger models tend to achieve higher binary judgment accuracy during early training, which may slightly mitigate the disadvantages of joint training but far from resolve them. Therefore, in future work, we plan to evaluate our approach on models with larger sizes, as well as on architectures such as MoE.

    \item The performance of M1-TTS is still strongly constrained by the editing capability of the underlying unified generative model. Our experiments show that although OmniVerifier-M1 can provide accurate bounding boxes and precise edit instructions, current image editing models are rarely trained to follow region-grounded editing commands. As a result, they may fail to restrict modifications to the specified bounding-box regions and instead introduce unnecessary or even harmful changes in unrelated areas. \textbf{This highlights an important direction for future research: developing fine-grained, region-level image editing models} that can faithfully execute localized instructions while preserving the rest of the image. Such grounding-aware editing capability is essential for enabling reliable dynamic image refinement and supporting more general, complex, and interactive generation scenarios.
\end{itemize}
\end{document}